\documentclass[conference]{IEEEtran}
\usepackage{times}
\IEEEoverridecommandlockouts

\makeatletter
\let\NAT@parse\undefined
\makeatother
\usepackage[square, sort&compress, comma, numbers]{natbib}

\usepackage{graphicx}        
\usepackage{multicol}        
\usepackage[bottom]{footmisc}
\usepackage{tikz}
\usepackage{subcaption} 
\usepackage{amsmath}
\usepackage{amsthm}
\usepackage{amssymb}
\usepackage{psfrag}
\usepackage{multirow}
\usepackage{url}
\usepackage{stmaryrd}
\SetSymbolFont{stmry}{bold}{U}{stmry}{m}{n}

\usepackage[ruled, vlined, linesnumbered]{algorithm2e}

\usepackage{array}
\usepackage{bm}
\usepackage{xspace}
\usepackage{xcolor, colortbl}
\usepackage{balance}

\newtheorem{definition}{Definition}

\newtheorem{proposition}{Proposition}

\DeclareGraphicsExtensions{.eps} 
\graphicspath{{Figures/}{./}}
\usepackage[mode=buildnew]{standalone}
\usepackage{tikz}

\title{Multi-Robot Path Deconfliction through Prioritization by Path Prospects}
\author{Wenying Wu$^*$, Subhrajit Bhattacharya$^\dagger$, Amanda Prorok$^*$
\thanks{*Wenying Wu and Amanda Prorok are with the University of Cambridge, UK: {\tt\small \{ww329|asp45\}@cam.ac.uk}. $^\dagger$Subhrajit Bhattacharya is with Lehigh University, USA: {\tt\small sub216@lehigh.edu}}
}

\begin{document}

\maketitle
\thispagestyle{empty}
\pagestyle{empty}

\begin{abstract}
This work deals with the problem of planning conflict-free paths for mobile robots in cluttered environments. Since centralized, coupled planning algorithms are computationally intractable for large numbers of robots, we consider decoupled planning, in which robots plan their paths sequentially in order of priority. Choosing how to prioritize the robots is a key consideration. State-of-the-art prioritization heuristics, however, do not model the coupling between a robot's mobility and its environment.
In this paper, we propose a prioritization rule that can be computed online by each robot independently, and that provides consistent, conflict-free path plans. Our innovation is to formalize a robot's \emph{path prospects} to reach its goal from its current location. 
To this end, we consider the number of homology classes of trajectories, and use this as a prioritization rule in our decentralized path planning algorithm, whenever any robots enter negotiation to deconflict path plans. This prioritization rule guarantees a partial ordering over the robot set. We perform simulations that compare our method to five benchmarks, and show that it reaches the highest success rate (w.r.t. completeness), and that it strikes the best balance between makespan and flowtime objectives.
\end{abstract}


\section{Introduction}

Technological advances are enabling the large-scale deployment of robots to solve various types of problems in logistics and transport, including product delivery~\cite{grippa:2017}, warehousing~\cite{enright:2011}, mobility-on-demand~\cite{pavone2012robotic}, and connected autonomous vehicles~\cite{prorok:2019}. 
Robot teams also hold the promise of delivering robust performance in unstructured or extreme environments~\cite{kantor2003distributed}. 
The commonality of many of these applications is that they require methods that assign and guide individual robots to goal locations on collision-free paths. The challenge of providing fast, optimal and complete solutions to this problem is very current, as we continue to complexify the problem domain by considering increasingly large and heterogeneous robot teams in navigation-constrained, cluttered environments. {In light of these developments, our work focuses on the \emph{coupling} between a robot's mobility traits and the built environment. In particular, we posit that a robot's ability to reach its goal can be measured, and that by integrating this measure in planning routines, better joint path plans can be found.}




Approaches to multi-robot path planning can generally be described as either centralized (assuming the existence of a central component that knows the state of the whole robot system) or decentralized (where no single component has the full picture, but cooperation must still be achieved). Centralized methods can be further categorized according to whether they are coupled or decoupled. Coupled approaches operate in the joint configuration space of all the robots, allowing for completeness (e.g., see~\cite{honig2016multi,wagner2011m}). However, solving for optimality is NP-hard~\cite{yu:2013}, and although significant progress has been made towards alleviating the computational load, e.g.,~\cite{wagner2011m, sharon2015conflict, ferner2013odrm}, these approaches still scale poorly in environments with a high number of path conflicts. 
On the other hand, decoupled approaches plan for each robot separately, and solve conflicts between paths as they arise, to ensure that collisions with other robots are avoided. Approaches to decoupled planning include sequential programming~\cite{chen:2015}, vehicle prioritization~\cite{vandenberg:2005} and velocity tuning~\cite{kant:1986}. These methods offer improved scalability, but often at the cost of completeness and optimality~\cite{Thrun02}.

\begin{figure}[t]
\centering
\includegraphics[width=0.9\columnwidth]{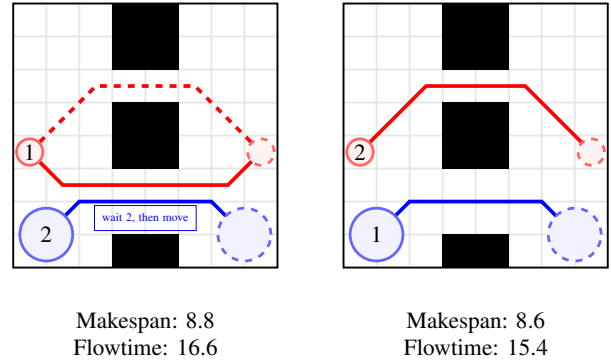}
\caption{An example problem where considering path diversity is important for the prioritization. The red robot has two possible paths, whereas the larger blue robot only has one. On the left, the red robot has first priority. It takes the shorter of its two paths, however this forces the blue robot to wait in place until it can follow. On the right, the blue robot, with lower path diversity, has first priority. The red robot is able to take its alternative path to avoid it, giving a faster overall solution. \label{fig:motivation}}
\end{figure}

Prioritized planning, first proposed in~\cite{Erdmann87} as a centralized strategy, is a very efficient method because it allows robots to plan sequentially in space-time in order of priority, eschewing the combinatorial complexities of coupled approaches. 
In this approach, each robot plans a minimum-cost path to its goal that avoids the computed trajectories of any higher-priority robots. Clearly, the chosen priority order will affect the solution found. It is generally desirable to optimize metrics such as makespan (the time at which the last robot in the team arrives), flowtime (the sum of all robots' travel times), or success rate (completion); targeting the optimization of either one of these objectives (but commonly not all simultaneously), researchers have proposed heuristics for choosing a planning order \cite{Turpin14,velagapudi2010decentralized,Thrun02, vandenberg:2005}.

The original prioritized planning idea in~\cite{Erdmann87} used a \textit{fixed} total priority ordering, and has been adapted in various papers to work in a decentralized manner (e.g.,~\cite{cap2013asynchronous,velagapudi2010decentralized}). However, choosing how to prioritize the team of robots still remains a key consideration. 
Moreover, as the operational conditions of the robots vary throughout time, and environments are in general not static, it is crucial to consider \emph{online} (dynamic) priority schemes. Although several dynamic priority schemes have been considered thus far (e.g., {\cite{azarm1997conflict, clark2002applying, deshpandereview, regele2006cooperative}}), none of these schemes account for the coupling between robot mobility and the environment, and hence, may fail to find better solutions. Such a scenario is exemplified in Fig.~\ref{fig:motivation}, which shows how considering a robot's path diversity leads to a reduction of both flowtime and makespan.


In this work, we focus on how the coupling between the built environment and a robot's mobility traits determines its path options to reach its goal. In specific, we propose a decentralized planning method that makes use of a novel prioritization rule based on an estimate of the robot's \emph{path prospects}. 
The key idea that underpins this method is that individual robots have distinct path prospects within the same environment, due to unique conditions arising from kinematic, dynamic, or environmental constraints. The purpose of this work is to provide a formal introduction to the concept of path prospects, and a demonstration of its utility in multi-robot path planning.
%
\section{Related Work}

Several papers have proposed prioritization heuristics and decentralized adaptations since prioritized path planning was first proposed in \cite{Erdmann87}. 


In ~\cite{vandenberg:2005}, Van Den Berg et al. propose a static heuristic based on the length of the path from a robot's start to its goal when ignoring the presence of other robots. Robots with a smaller path length are given lower priority; the intuition is that they can afford to spend more time planning around other robots without impacting the makespan of the overall solution. 
In a similar approach, the work in~\cite{velagapudi2010decentralized} uses a robot's planning time (rather than path length) to determine static priorities.
This approach is generalized in~\cite{cap2013asynchronous} to account for asynchronous communications. Although the aforementioned approaches are based on decentralized computational models, they use static prioritization methods and rely on global knowledge.
%


The algorithm proposed in \cite{azarm1997conflict} considers an \emph{online} prioritization method, whereby robots with conflicting paths consider all possible priority orders, and choose the best one. Similarly, the work in~\cite{ma2018searching} evaluates space of all possible priority orderings in a conflict-driven combinatorial search framework.
However, this style of exhaustive negotiation does not scale to conflicts beyond a small number of robots, since for $N$ robots there exist $N!$ different priority orders. 

Regele et al.~\cite{regele2006cooperative} define a method whereby a robot can raise its own priority if it detects that it will become blocked by another robot. The disadvantage is that the first robot to raise its priority and demand right of way will usually obtain it; the paper states that it is difficult to predict which solution will be chosen by the algorithm in a given situation since just a small time difference in the execution of a robot's plan can result in a completely different solution.
The algorithm proposed in \cite{clark2002applying} (and also its extension in~\cite{deshpandereview}) has every robot maintain a list of the robots currently within its field of view. 
Replanning is triggered whenever a new robot comes into range, and a robot avoids the paths of higher priority robots in its list when planning. 
%
The authors use a dynamic heuristic based on the current local workspace; they define a function which counts the number of environmental obstacles within some range of the robot, and allow robots whose workspace is more crowded to have higher priority. Although this work has some similarity to our approach in that it considers how environmental clutter might hinder a robot's ability to reach its goal efficiently, it fails to truly model the paths available to a robot, given the robot's specific mobility traits and its motion constraints in the surrounding environment.

\textbf{Contributions.} Overall, none of these existing approaches use heuristics or schemes that explicitly account for the coupling between a robot's mobility and the environment. This work focuses on decentralized multi-robot path planning. Within this context, our main contribution is a novel prioritization heuristic, based on the number of robots' \emph{path prospects}, that implicitly takes into account the coupling between environment and mobility traits. 
To this end, we develop a prioritization rule that has two key components: \textbf{(1)} a method that estimates the \emph{number of path options} a robot has to reach its goal, based on theory from algebraic topology, and \textbf{(2)}, a method that defines the \emph{area of relevance}, within which these path options are computed. 
We show how our prioritization rule is embedded in a decentralized, dynamic planning algorithm to de-conflict robot trajectories. 
We prove that this dynamic planning algorithm provides a partial ordering over the robot set, and hence, is cycle-free. Our results demonstrate that the planned solutions provide very competitive makespan and flowtime performance; moreover, they provide the  best trade-off between these two conflicting objectives.


%
\section{Problem Description}

We consider a $D-$dimensional workspace $\mathcal{W} \subseteq \mathbb{R}^D$ and a set of $B$ static obstacles $\mathcal{O} = \{o_1,\dots,o_B\}$ with $o_i \subset \mathcal{W}$. A team of $N$ robots $\mathcal{R} = \{r_1, \ldots, r_N\}$ navigate in this shared workspace.
The robot team is heterogeneous in size; the effective space occupied by robot $r_n$ is referred to by $\rho_n$. 

\textit{Graph representation.} Each robot travels along the edges of a directed graph $G_n = \langle \mathcal{V}_n, \mathcal{E}_n \rangle$, which allows only feasible motion and accounts for all constraints (morphological, kinematic, dynamic). In particular, a robot $r_n$ that travels along edges in $G_n$ cannot collide with any obstacles in $\mathcal{O}$.
The set $\mathcal{V}_n$ is defined by vertices $v_i = \langle \mathbf{x}_i, t_i \rangle$ with $\mathbf{x}_i \in \mathcal{W}$ and $t_i \in \mathbb{R}^+$. The set $\mathcal{E}_n$ is defined by directed edges $e_{ij}: \mathbb{R}^+ \mapsto \mathbb{R}^D$, between vertex $v_i$ and $v_j$, such that $e_{ij}(t_i) = \mathbf{x}_i$, $e_{ij}(t_j) = \mathbf{x}_j$, and $t_i \leq t_j$. In other words, the graph $G_n$ exists in a $(D+1)$-dimensional space, where the last dimension represents time.

\textit{Labeled assignment}. Robot $r_n$ is assigned a start location $\mathbf{s}_n \in \mathcal{W}$ (corresponding to vertex $v_i$ with $\mathbf{x}_i = \mathbf{s}_n$ and $t_i=0$).
Similarly, robot $r_n$ is assigned a goal location $\mathbf{g}_n \in \mathcal{W}$ (corresponding to a set of vertices $v_i$ with $\mathbf{x}_i = \mathbf{g}_n$ and $t_i \in \mathbb{R}^+$). A labeled assignment $\mathcal{A}$ is a set of tuples $\{\langle \mathbf{s}_1, \mathbf{g}_1 \rangle, \ldots, \langle \mathbf{s}_N, \mathbf{g}_N \rangle\}$, for all robots in $\mathcal{R}$. 


\textit{Conflict-free trajectories}. A robot $r_n$ has a trajectory $\pi_n: \mathbb{R}^+ \mapsto \mathcal{W}$ that represents a sequence of edges traversed in $G_n$ such that two consecutive edges share a common vertex. 
A trajectory $\pi_n$ is said to be \emph{satisfying} if $\pi_n(0) = \mathbf{s}_n$ and there exists a time $t_n^f$ such that $\pi_n(t_n^f) = \mathbf{g}_n$.
A robot $r_n$ navigating along this path defines a volume $V(\pi_n, \rho_n)$ that depends on its size. To coordinate the navigation in $\mathcal{W}$, two robots $r_n$ and $r_m$ can share their path plans with each other if they are within communication range, i.e., if their positions are separated by a quantity less than $c$.
We make use of a function $\textsc{Trim} (G_n, \rho_n, V(\pi_m, \rho_m))$ that removes all unfeasible paths in $G_n$ that would collide with the volume defined by robot $r_m$. Any path in the graph returned by $\textsc{Trim}$ is ensured to be \emph{conflict-free} with the path $\pi_m$ of robot $r_m$.

In order to facilitate the definition of a given robot's configuration space, we define the notion of an \emph{effective obstacle}, which is a set of original obstacles in $\mathcal{O}$, such that no trajectories in a given graph passes between them (see Figure~\ref{fig:prioritization}). Specifically, a robot $r_n$ has a set of effective obstacles $\mathcal{\tilde{O}}_n = \{\tilde{o}_1,\ldots,\tilde{o}_{\tilde{B}}\}$, $\tilde{B}\leq B$, with $\tilde{o}_i \subseteq \mathcal{O}$ and $\cup_i \tilde{o}_i = \mathcal{O}$ and $\cap_i \tilde{o}_i = \emptyset$.

Figure~\ref{fig:graph_plane} shows a labeled assignment for two robots that must plan minimum-cost trajectories from their start positions to their goal positions. Figure~\ref{fig:graph_volume} demonstrates how robot $r_2$ circumnavigates the path plan of robot $r_1$, after execution of $\textsc{Trim} (G_2, \rho_2, V(\pi_1, \rho_1))$.

\textit{Assumptions}.
We assume that a robot is able to check for collisions between its own planned path and another robot's. To facilitate this, we assume all their clocks are synchronized. Messaging delay can be accommodated, however, it must be negligible with respect to robot dynamics (i.e., the speed at which the motion graph is traversed). We assume that robot detections are always mutual (when they come into communication range).

\textit{Objective}. Our goal is to find a method that strikes the best balance between minimizing the mean flowtime ($\sum_n t_n^f/N$) and minimizing the makespan ($\max_n t_n^f$), such that each robot $r_n$ follows a satisfying trajectory $\pi_n$ which is conflict-free with all other robots' paths. 
We note that, in general, these objectives demonstrate a pairwise Pareto optimal structure, and cannot be simultaneously optimized~\cite{yu:2013}.


\begin{figure}[tb]
\centering
\psfrag{a}[cc][][0.95]{$\mathbf{s}_1$}
\psfrag{e}[cc][][0.95]{$\mathbf{g}_1$}
\psfrag{i}[cc][][0.95]{$\mathbf{g}_2$}
\psfrag{o}[cc][][0.95]{$\mathbf{s}_2$}
\psfrag{x}[cc][][1]{$x$}
\psfrag{y}[cc][][1]{$y$}
\includegraphics[width=0.8\columnwidth]{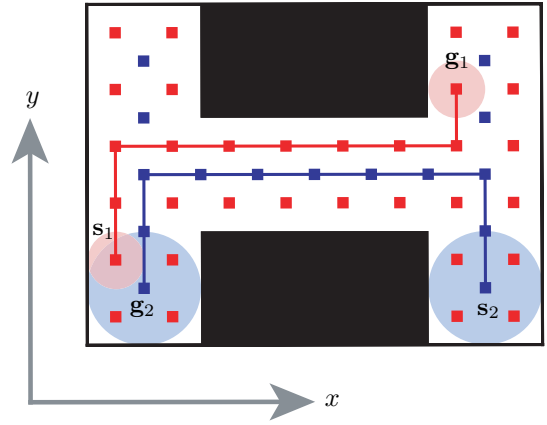}
\caption{Planar workspace with two robots, $r_1$ and $r_2$, and their respective start and goal positions. Robot $r_2$ has an effective size $\rho_2$ that is twice that of robot $r_1$. The minimum-cost paths would result in a collision. 
\label{fig:graph_plane}}
\end{figure}

\begin{figure}[tb]
\centering
\psfrag{x}[cc][][1]{$x$}
\psfrag{y}[cc][][1]{$y$}
\psfrag{t}[cc][][0.95]{$t$}
\psfrag{v}[lc][][0.95]{$V(\pi_1, \rho_1)$}
\psfrag{u}[lc][][0.95]{$V(\pi_2, \rho_2)$}
\psfrag{p}[rc][][0.95]{$\pi_1$}
\psfrag{e}[cc][][0.95]{$\mathbf{g}_1$}
\psfrag{i}[cc][][0.95]{$\mathbf{g}_2$}
\includegraphics[width=0.8\columnwidth]{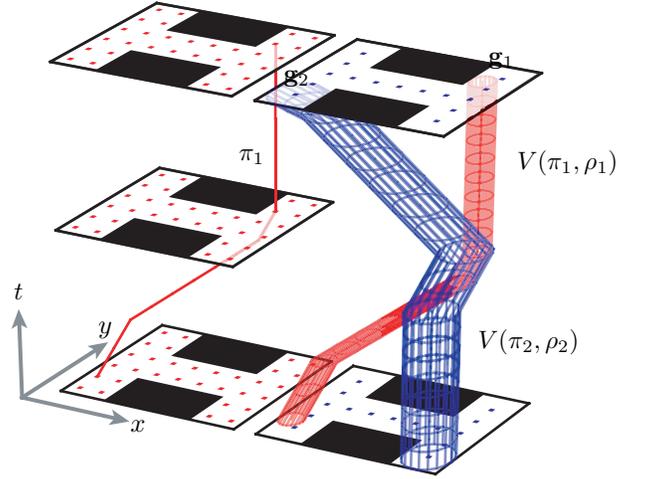}
\caption{On the left, we plot the space-time graph $G_1$ with a minimum-cost trajectory $\pi_1$ for robot $r_1$. On the right, we see how trajectory $\pi_2$ sweeps a volume $V(\pi_2, \rho_2)$ that does not intersect with $V(\pi_1, \rho_1)$.  \label{fig:graph_volume}}
\end{figure}

\section{Decentralized Coordination}

Our decentralized path planning algorithm can be broken down into two levels: at the higher level (i.e., coordinated planning), we consider how robots communicate and negotiate a priority ordering; at the lower level (i.e., individual planning), we consider how an individual robot (re)plans a trajectory to its goal given its current knowledge about the environment and the plans of other robots within communication range. We make use of the following definitions.

\begin{algorithm}[tb]
\caption{Dynamic Prioritized Path Planning \label{alg:dynamic}}
\DontPrintSemicolon
$H_n \gets \emptyset$ // list of higher priority robots \;
$L_n \gets \emptyset$ // list of lower priority robots \;
$\pi_n \gets \textsc{ComputeNewPlan}(H_n)$ \;
\While{\textsc{True}} {
    \If{new robot $r_m$ (with priority $\xi_m$) in range $c$ \emph{\bf or} received new priority $\xi_m$ from $r_m$} {
        $\xi_n \gets \textsc{ComputePriority}(r_n)$\;
        $H_{n,old} \gets H_n$ and  $L_{n,old} \gets L_n$\;
        $H_{n} \gets \{r_i | \xi_i \prec \xi_n \wedge r_i \in H_{n,old} \cup L_{n,old} \cup \{r_m\} \} $\;
        $L_{n} \gets \{r_i | \xi_i \succ \xi_n \wedge r_i \in H_{n,old} \cup L_{n,old} \cup \{r_m\} \} $\;
        \If{$H_{n,old} \neq H_n$}{
            $\pi_n \gets \textsc{ComputeNewPlan}(H_n)$\;
        }
    }
    \ElseIf{$r_m$ left range $c$} {
        $L_n \gets L_n \setminus r_m$\;
        \If{$r_m \in H_n$}{
            $H_n \gets H_n \setminus r_m$\;
            $\pi_n \gets \textsc{ComputeNewPlan}(H_n)$\;
        }
    }  
    \ElseIf{receive new plan $\pi_m$ from $r_m$}{
        \If{$r_m \in H_m$}{
            $\pi_n \gets \textsc{ComputeNewPlan}(H_n)$\;
        }
    }    
}
\end{algorithm}

\begin{algorithm}[tb]
\caption{Re-plan trajectory \label{alg:function_plan}}
\DontPrintSemicolon
\SetKwInput{fun}{Function}
\fun{\textsc{ComputeNewPlan}$(H_n)$}
$G \gets G_n$ \;
\For{$r_m \in H_n$} {
    $\textsc{ReceivePlanFrom}(r_m)$\;
    $G \gets \textsc{Trim}(G, \rho_n, V(\pi_m, \rho_m))$\;
}
$\pi_n \gets$ \textsc{PlanPathFromCurrentPosition}$(G, \mathbf{g}_n)$\;
$\xi_n \gets \textsc{ComputePriority}(r_n)$\;
$\textsc{BroadcastPlanIfChanged}(\pi_n)$\;
$\textsc{BroadcastPriorityIfChanged}(\xi_n)$\;
\Return $\pi_n$
\end{algorithm}


\begin{definition}[Priority ordering]
A priority ordering $\prec$ is such that a robot $r_n \in \mathcal{R}$ with priority $\xi_n$ is of higher priority than robot $r_m$ with priority $\xi_m$ iff $\xi_n \prec \xi_m$. 
\end{definition}

\begin{definition}[Ordered robot set]
Given a priority ordering $\prec$ on a set of robots $\mathcal{R}$, the pair $(\mathcal{R}, \prec)$ is a strict partially ordered robot set. 
\end{definition}

\begin{definition}[Ordered robot neighborhood]
Given a priority ordering $\prec$, for a given robot $r_n$, $H_n = \{r_m | \xi_m \prec \xi_n\}$ is the set of robots with higher priority, and $L_n = \{r_m | \xi_m \succ \xi_n\}$ is the set of robots with lower priority. The neighborhood of robot $r_n$ defined as $\mathcal{N}_n = H_n \cup L_n \cup \{r_n\}$ is strongly connected (by symmetry of communication). 
By definition, the robot neighborhood $\mathcal{N}_n$ is an ordered robot set $(\mathcal{N}_n, \prec)$.
\end{definition}

\subsection{Coordination Strategy}
Our coordination strategy is detailed in Algorithm~\ref{alg:dynamic}, and is based on two main elements, described as follows:

\textit{Computation of an ordered robot neighborhood}: 
Each robot is able to detect other robots when they come into the communication range $c$, and when they leave it. A robot $r_n$ has a priority score $\xi_n$, which it can compute independently by a function \textsc{ComputePriority} (see Section~\ref{sec:path_prospects}).
Each robot $r_n$ maintains two lists of robots currently in its range:  $H_n$ contains the robots with higher priority whilst $L_n$ contains the robots with lower priority. In a dynamic priority scheme, $r_n$ recomputes $\xi_n$ whenever a new robot comes into range. It then broadcasts this updated priority value to ensure all robots within range (i.e., in its neighborhood) have a consistent plan.

\textit{Re-planning}: Re-planning is triggered for $r_n$ in three cases: \textit{(i)} when a new robot comes into range that has a higher priority, \textit{(ii)} when an updated plan is received from a higher priority robot, or \textit{(iii)}, when a robot $r_m$ broadcasts a new priority $\xi_m$.  When $r_n$ re-plans, it calls a function \textsc{ComputeNewPlan} that takes into account the planned trajectories of robots with higher priority (in $H_n$). The robot then communicates its new plan to robots in $L_n$. 

\begin{proposition}
Algorithm~\ref{alg:dynamic} is deadlock-free. 
\end{proposition}
\begin{proof}
Since each robot $r_n$ in $\mathcal{R}$ executes Algorithm~\ref{alg:dynamic}, the result is a collection of ordered robot neighborhoods $\mathcal{N}_n$, $\forall n$. 
If two robot neighborhoods $\mathcal{N}_n$ and $\mathcal{N}_m$ share a common robot $r_j$, then, by transitivity, there must be a partial ordering in $\mathcal{N}_n \cup \mathcal{N}_m$, since Algorithm~\ref{alg:dynamic} ensures that robot $r_j$ can only have one priority score $\xi_j$ that is broadcast. Hence, Algorithm~\ref{alg:dynamic} constructs an ordered robot set \mbox{$(\mathcal{R}, \prec)$}. Since partial orderings are acyclic, no planning deadlocks can arise.
\end{proof}




\subsection{Individual Robot Planning}
\label{sec:individual_robot_planning}

Each robot handles the computation of its own minimum-cost trajectory from its current location to its goal location $\mathbf{g}_n$ (see function \textsc{ComputeNewPlan}). The resulting trajectory avoids the static obstacles in the environment as well as the planned paths of any higher-priority robots. 
In our implementation, all robots use the HCA* algorithm proposed in \cite{Silver2005} which applies A* search to a space-time map, and uses a reservation table to record the trajectories of other robots to be avoided. This effectively implements the function \textsc{Trim}. The complexity of TRIM is $O(|\pi_m| log|E_n|)$ assuming the usage of a fast spatial lookup structure, such as a quadtree.

Our approach is general, in that any path planning algorithm that is able to avoid dynamic obstacles with known trajectories can be used; indeed it is even possible for different robots to use different algorithms so long as an implementation of the function \textsc{Trim}, which reconciles heterogeneous space-time graphs, can be embedded into the planning function.


\section{Prioritization Based on Path Prospects}
\label{sec:path_prospects}

During navigation, when robots come within communication range, they enter negotiations to deconflict their path plans. To facilitate this negotiation, we implement a rule that prioritizes robots with \emph{fewer path options}. Our prioritization rule has two key components: \textbf{(1)} a method that estimates the \emph{number of options} a robot has to reach its goal, and \textbf{(2)}, a method that defines the \emph{area} within which these path options are computed.
The following paragraphs detail our approach. 

\subsection{Homology Classes of Trajectories}
To develop a method for \textbf{(1)} above, we build on theory from algebraic topology.
For a particular robot, we consider the trajectories in different \emph{homology classes} as the path prospects for the robot.
Homology classes (of trajectories) in an environment represent topologically distinct classes of trajectories (Figure~\ref{fig:path_prospects}).
Two trajectories connecting the same start and goal points on a planar domain are said to be in the same homology class if the closed loop formed by the two trajectories are \emph{null homologous}, i.e., it forms the oriented boundary of a two-dimensional obstacle-free region~\cite{planning:AURO:12}.
The homology class of a loop, in turn, can be quantified by winding numbers around the connected components of effective obstacles, with the null homologous class having zero winding number around every obstacle. Thus, in a planar domain with $z$ connected components of effective obstacles, the \emph{homology invariant} of a loop is given by a vector of integers, $[h_1, h_2,\cdots,h_z]\in\mathbb{Z}^z$, where $h_i$ is the winding number around the $i^{\text{th}}$ obstacle.

\begin{figure}[tb]
    \centering
    \includegraphics[width=0.7\columnwidth]{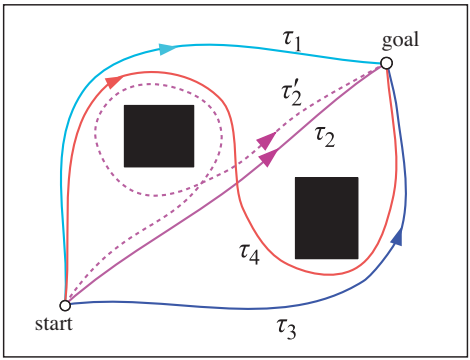}
    \caption{Homology classes of trajectories. $\tau_2$ and $\tau'_2$ are in different classes in regular homology, but map to the same class in $\mathbb{Z}_2$-coefficient homology.} \label{fig:hom-classes}
\end{figure}

However, there are infinitely many homology classes since a trajectory can loop/wind around the same obstacle multiple times, and for every different number of windings the class assigned to the loop is different.
In order to prevent the separate counting of the multi-looping homology classes, one can compute the homology invariants in the ``mod $2$'' coefficient~\cite{Persistence-Plnning:TRO:15}, $\mathbb{Z}_2 = \mathbb{Z} / 2\mathbb{Z} = \{0,1\}$. Simply put, the homology invariant in the $\mathbb{Z}_2$ coefficients become $[h_1, h_2,\cdots,h_z] \mod 2 \in\mathbb{Z}_2^z$. Doing so identifies all the even winding numbers to $0$ and all the odd winding numbers to $1$, thus preventing the creation of separate homology classes for loops that wind around obstacles multiple times (Figure~\ref{fig:hom-classes}).
$\mathbb{Z}_2^z$ is a finite set, and in fact has $2^z$ elements. Thus, the number of $\mathbb{Z}_2$ coefficient homology classes in a planar domain with $z$ connected components of effective obstacles is $2^z$, which we use in the construction of heuristics in the path prospect algorithm.

\subsection{Path Prospect Algorithm}
We use the number of $\mathbb{Z}_2$ coefficient homology classes in an area with $z$ effective obstacles to return an estimate of a robot $r_n$'s \emph{path prospects} $P_{n}^{(t)}$ at time $t$ in that area. Next, we develop a method for computing a relevant area (and its associated vertices), to define the component \textbf{(2)}, above.

A robot $r_n$'s \emph{path prospects} $P_{n}^{(t)}$ are an indicator of the number of distinct paths to goal $\mathbf{g}_i$ from its current location at time $t$. This can be estimated by counting the effective obstacles $\mathcal{\tilde{O}}_n$ which $r_n$ will likely come across as it moves \emph{towards} its goal $\mathbf{g}_n$ from its current position.
Specifically,  we do not wish to count any effective obstacles that lie {behind} the robot, given $\mathbf{g}_n$ and its current location. To achieve this, we define a set of \textit{forwards vertices} $F_n^{(t)} \subseteq \mathcal{V}_n$ and count only the effective obstacles whose areas intersect the area in $\mathcal{W}$ containing all vertices in $F_n^{(t)}$ and the edges that link them.

\begin{figure}[t]
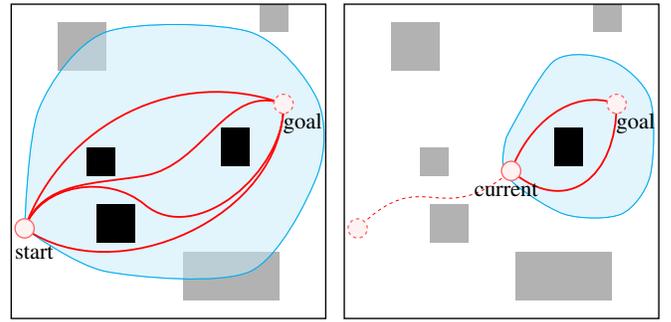

    \centering
        \begin{subfigure}{0.5\columnwidth}
            \centering
            \includegraphics[width=0.95\textwidth]{Figures/path_prospects_1.tex}
            \caption{At $t = 0$, $P_{n}^{(0)} = 2^3$.}
    \end{subfigure}%
    \begin{subfigure}{0.5\columnwidth}
            \centering
            \includegraphics[width=0.95\textwidth]{Figures/path_prospects_2.tex}
            \caption{At $t = 17$, $P_{n}^{(17)} = 2^1$.}
    \end{subfigure}
    \caption{Illustration of path prospects for a robot navigating to its goal, computed for two different moments in time. In (a), only 4 representative paths out of 8 are shown, for clarity. }
    \label{fig:path_prospects}
\end{figure}
\begin{figure}[t]
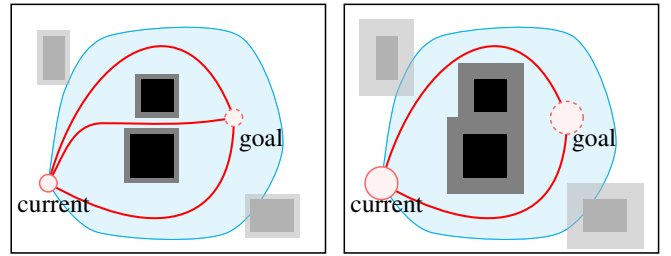

    \centering
    \begin{subfigure}{0.5\columnwidth}
            \centering
            \includegraphics[width=0.95\textwidth]{Figures/path_prospects_3.tex}
            \caption{$P_{n}^{(t)} = 2^2$.}
    \end{subfigure}%
    \begin{subfigure}{0.5\columnwidth}
            \centering
            \includegraphics[width=0.95\textwidth]{Figures/path_prospects_4.tex}
            \caption{$P_{m}^{(t)} = 2^1$.}
    \end{subfigure}
    \caption{Example where two robots with different sizes have different path prospects. In (b), the two central obstacles merge into a single effective obstacle. The lighter borders around each obstacle depict their inflation by the robot’s radius $\rho_n$, which is one method for computing effective obstacles.}
    \label{fig:prioritization}
\end{figure}

To define $F_n^{(t)}$, we use the notion of true distance as proposed in~\cite{Silver2005}. The true distance of a vertex $v \in \mathcal{V}_n$ is the length of the shortest satisfying path in $G_n$ from $v$ to $\mathbf{g}_n$. We define $F_n^{(t)} \subseteq \mathcal{V}_n$ to be the set of vertices that are reachable from $r_n$'s location at time $t$ by only transitioning from a vertex $v_i$ to a vertex $v_j$, if $v_j$ can still lead to paths that are shorter than the estimated longest true distance of the robot team. 
This can be computed using a variant of Dijkstra's algorithm (see Algorithm \ref{alg:prospects}).
Note that the longest true distance can be estimated locally by broadcasting \textsc{TrueDistance}$(\mathbf{s}_n, \mathbf{g}_n)$ along with priority $\xi_n$ in Algorithm~\ref{alg:dynamic}.
Figure~\ref{fig:path_prospects} illustrates the path prospects for a robot navigating towards its goal, at two consecutive moments in time.

\subsection{Prioritization Heuristic}

We use the path prospect algorithm (Algorithm~\ref{alg:prospects}) to prioritize robots with conflicting paths. For robots $r_n$ and $r_m$, we define the ordering $\prec$ such that
\begin{equation}
\label{eq:priority}
P_n^{(t)} < P_m^{(t)} \Leftrightarrow \xi_n \prec \xi_m.
\end{equation}
Priority orderings are negotiated through Algorithm~\ref{alg:dynamic}. By prioritizing robots that have fewer path prospects, we force those robots that have more options to deviate from their preferred (best) plans, and to give way to the robots that have fewer options. Figure~\ref{fig:prioritization} illustrates how different robot sizes affect the available path prospects (and hence the priority ordering). 


\begin{algorithm}[tb]
\DontPrintSemicolon
\caption{Path Prospects \label{alg:prospects}}
\SetKwInOut{Input}{Input}\SetKwInOut{Output}{Output}
\Input {current position of $r_n$: $v_n$, goal location $\mathbf{g}_n$, untrimmed graph $G_n$, effective obstacles $\mathcal{\tilde{O}}_n$, estimated longest path length $T$}
\Output {path prospects $P_n^{(t)}$}
    $F_{n}^{(t)} \gets \textsc{GetForwardsVertices}(v_n, \mathbf{g}_n, G_n, T)$\;
    $A \gets \textsc{ComputeArea}(F_n^{(t)},\mathcal{E}_n)$\;
    $\kappa \gets 0$\;
    \For{$o \in \mathcal{\tilde{O}}_n$}{
        \If{$o \cap A = o$}{
            $\kappa \gets \kappa + 1$ // count this obstacle\;
        }
    }
    \Return $2^{\kappa}$
\end{algorithm}

\begin{algorithm}[tb]
\DontPrintSemicolon
\SetKwInput{fun}{Function}
\caption{Compute Set of Forwards Vertices \label{alg:forwards}}
\fun{\textsc{GetForwardsVertices}$(v, \mathbf{g}, G, T)$}{
    visited $\gets \emptyset$\;
    priority\_queue $\gets$ \{$v$\}  // prioritizes by smallest $t$\;
    \While{\upshape{priority\_queue} $\neq \emptyset$}{
        $q \gets$ \textsc{PopSmallest}(priority\_queue) with $q = \langle\mathbf{x}_q, t_q\rangle$\;
        \If{$\mathbf{x}_q \notin$ \upshape{visited}}{
            neighbours $\gets \textsc{FindNeighbours}(G, q)$\;
            \For{$n \in$ \upshape{neighbours with} $n = \langle\mathbf{x}_n, t_n\rangle$}{
                \If{$t_n + \textsc{TrueDistance}(n,\mathbf{g}) \leq T$}{
                    \textsc{Append}(priority\_queue,$n$)\;
                }
            }
            visited $\gets$ visited $\cup \{\mathbf{x}_q\}$\;
        }
    }
    \Return visited
}
\end{algorithm}

\section{Evaluation}
We implement our method in grid-worlds. This allows us to easily create valid graphs $G_n$ for all robots, implement the corresponding \textsc{Trim} function, and create a set of effective obstacles $\mathcal{\tilde{O}}_n$ for any robot $r_n$ by inflating original obstacles in $\mathcal{O}$ by $\rho_n$. We note that this dilation can be done more generally (beyond regular grid-worlds) by applying Minkowski addition.

We evaluate the performance of our method in two experiments. The first experiment (\textbf{S1}) tests the method across different types of environment. We generate six different cluttered grid-worlds, depicted in Figure~\ref{fig:maps}, of size 75$\times$75. We use a team of 10 robots of five different sizes, with two robots per size, and sizes ranging from 1 to 5. For each base environment, we generate 500 problems (random assignments), and record the performance of the solutions provided by our algorithm (with two alternative tie-break options to guarantee strict orderings), as well as by five additional benchmark algorithms (described below). We solve each problem across communication ranges $c$ that vary between 30 and 50.

The second experiment (\textbf{S2}) tests the performance of our method in a large environment with a large number of robots. We quadruple environment Maze-1 (Fig.~\ref{fig:maps} (a)) to produce a map of size 150$\times$150. We use a team of 100 robots of size ranging from 1 to 4, in equal proportion, with a communication range of 50. We generate 500 problems and record the performance of all seven algorithms (as above).


\begin{figure}
    \centering
    \begin{subfigure}{0.3\columnwidth}
            \centering
            \includegraphics[width=0.95\textwidth]{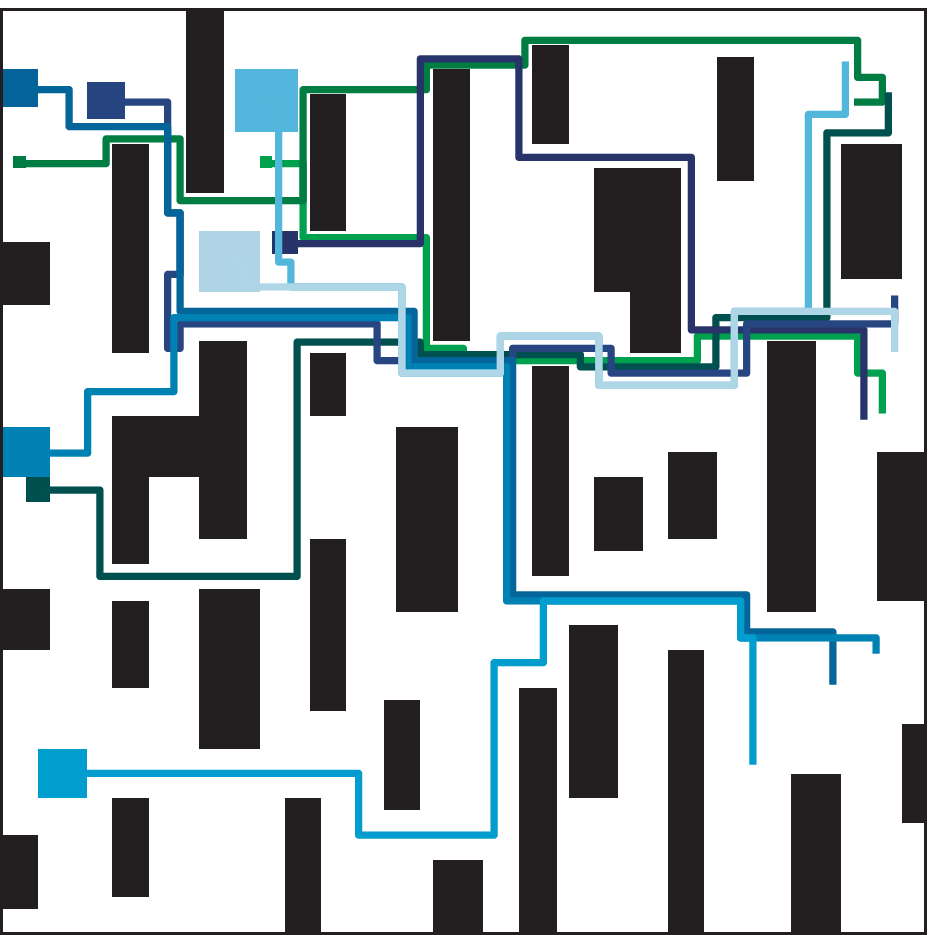}
            \caption{Maze-1}
        \label{fig:prospects_a}
    \end{subfigure}%
    \begin{subfigure}{0.3\columnwidth}
            \centering
            \includegraphics[width=0.95\textwidth]{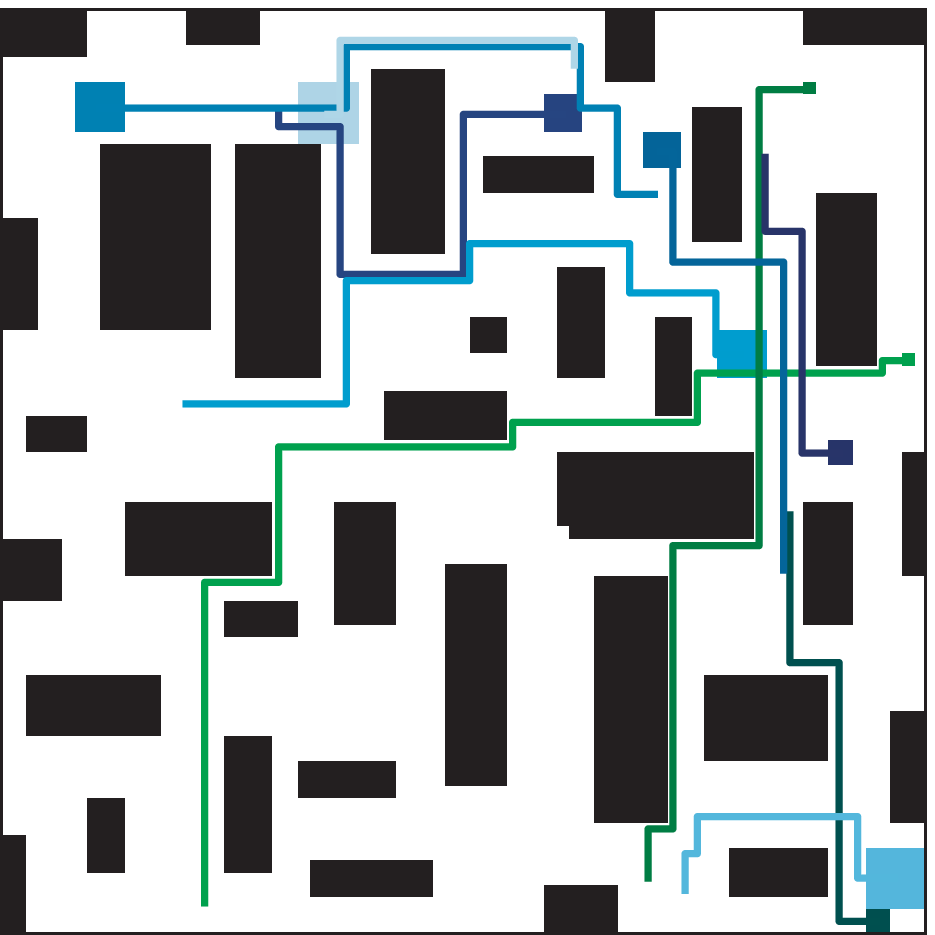}
            \caption{Maze-2}
        \label{fig:prospects_b}
    \end{subfigure}%
    \begin{subfigure}{0.3\columnwidth}
            \centering
            \includegraphics[width=0.95\textwidth]{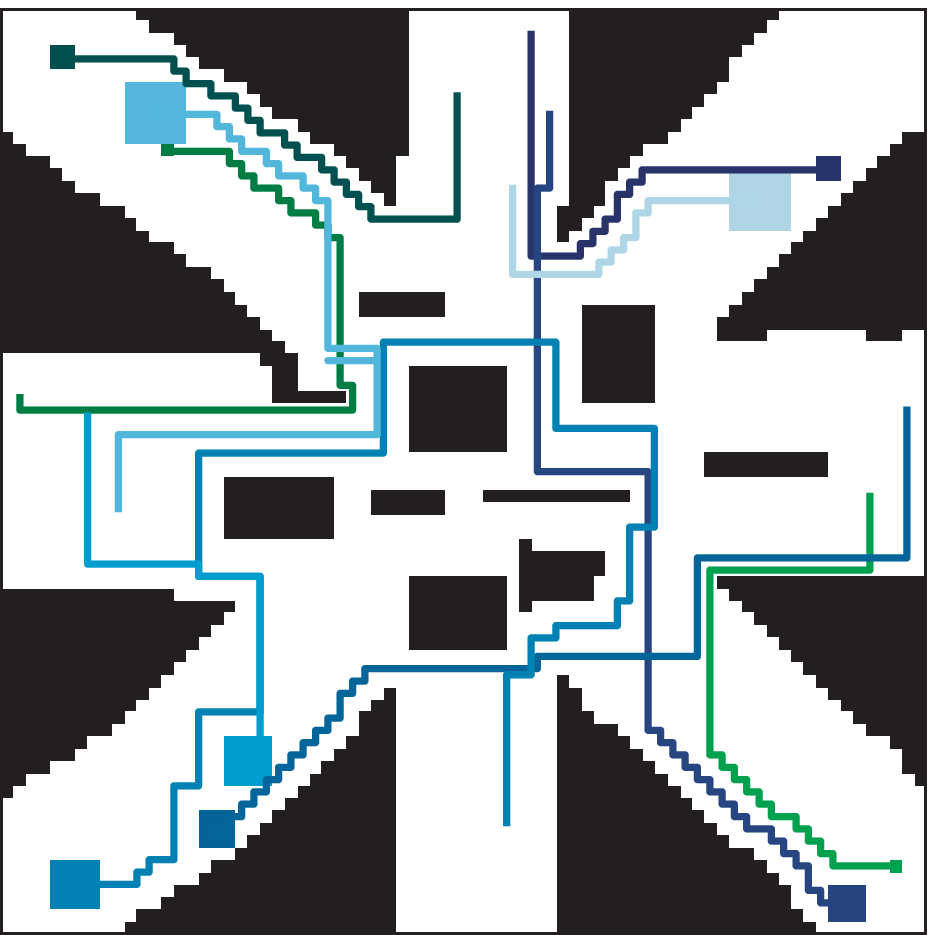}
            \caption{Crossing}
        \label{fig:prospects_c}
    \end{subfigure} 
    \\
    \begin{subfigure}{0.3\columnwidth}
            \centering
            \includegraphics[width=0.95\textwidth]{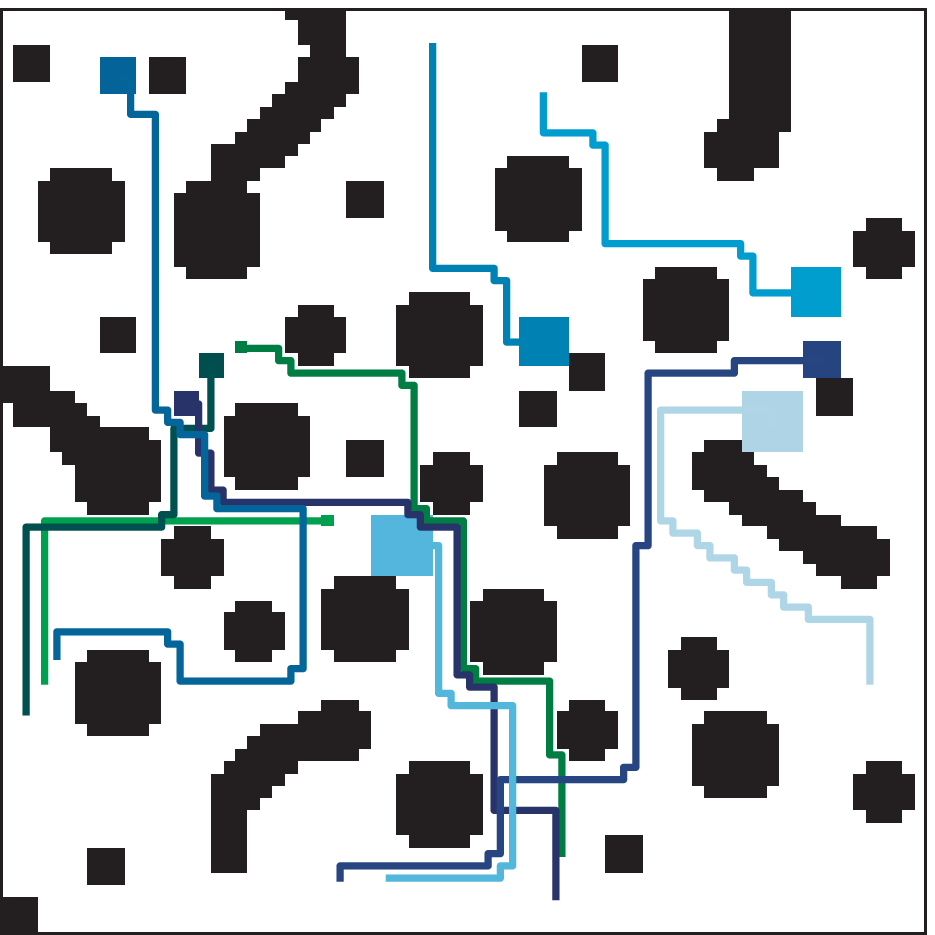}
            \caption{Clutter}
        \label{fig:prospects_d}
    \end{subfigure}
    \begin{subfigure}{0.3\columnwidth}
            \centering
            \includegraphics[width=0.95\textwidth]{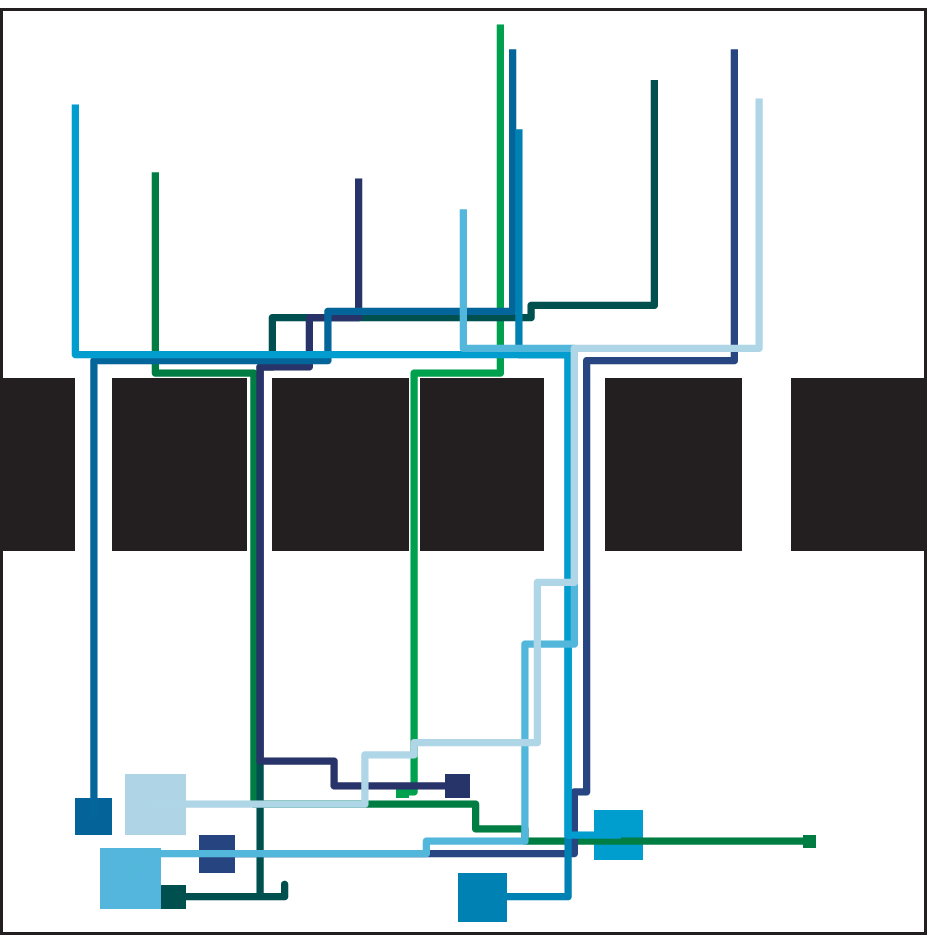}
            \caption{Corridor}
        \label{fig:prospects_e}
    \end{subfigure} 
    \begin{subfigure}{0.3\columnwidth}
            \centering
            \includegraphics[width=0.95\textwidth]{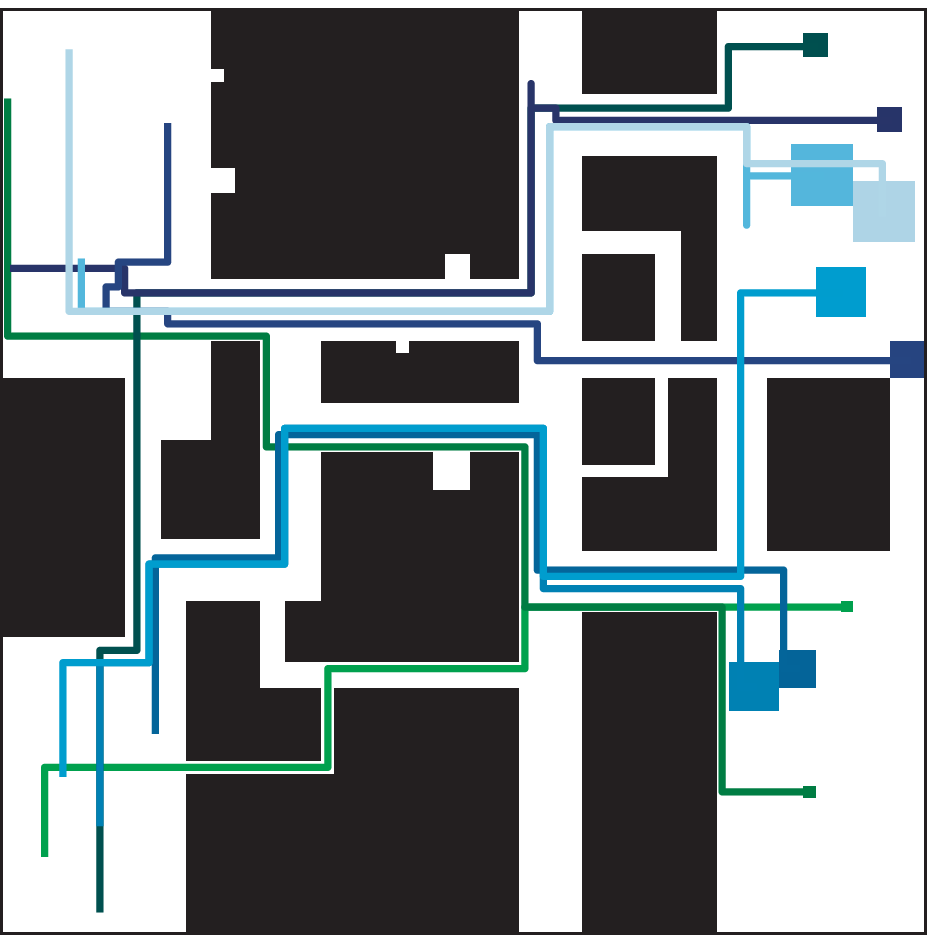}
            \caption{Tunnel}
        \label{fig:prospects_f}
    \end{subfigure} 
    \caption{Examples of path solutions (blue lines) for the six maps used in our problem sets. In each problem, 10 robots (blue squares) of five different sizes are assigned random start and goal positions.}
    \label{fig:maps}
\end{figure}

\begin{figure*}
    \centering
    \begin{subfigure}{0.32\textwidth}
            \centering
            \includegraphics[width=0.95\textwidth]{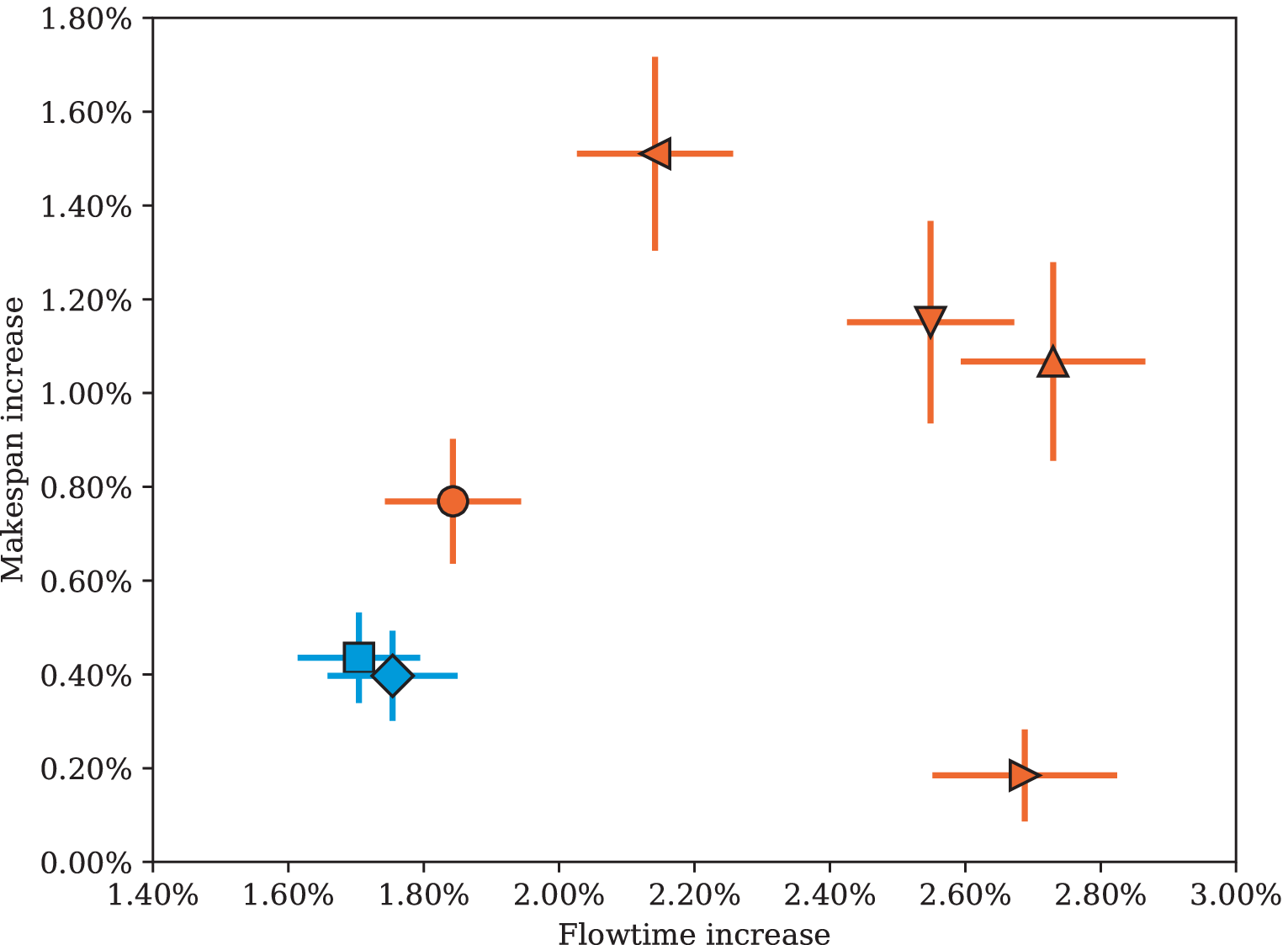}
            \caption{Clutter}
        \label{fig:p_a}
    \end{subfigure}%
    \begin{subfigure}{0.32\textwidth}
            \centering
            \includegraphics[width=0.95\textwidth]{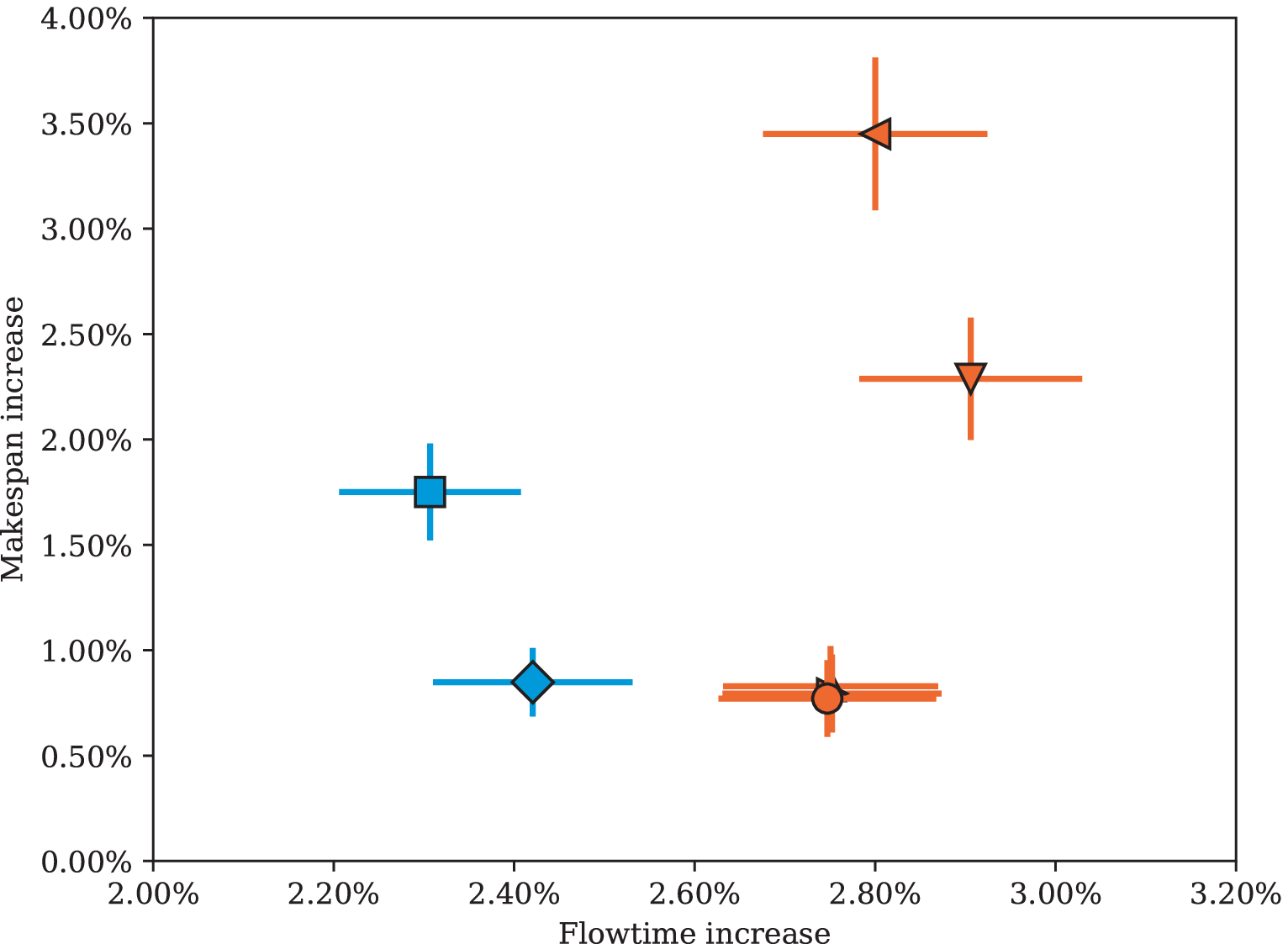}
            \caption{Corridor}
        \label{fig:p_b}
    \end{subfigure} 
    \begin{subfigure}{0.32\textwidth}
            \centering
            \includegraphics[width=0.95\textwidth]{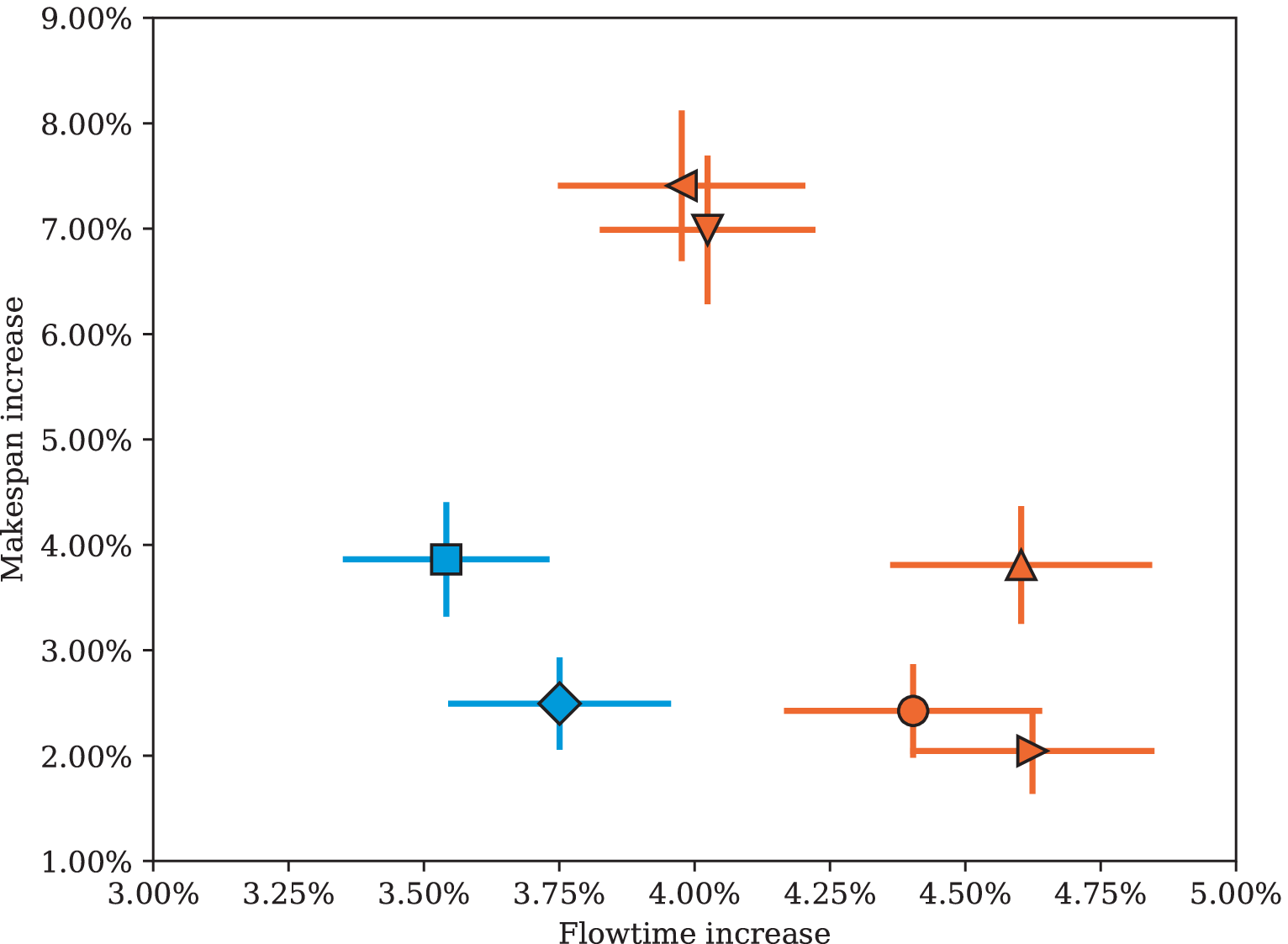}
            \caption{Crossing}
        \label{fig:p_c}
    \end{subfigure} \\
    \begin{subfigure}{0.32\textwidth}
            \centering
            \includegraphics[width=0.95\textwidth]{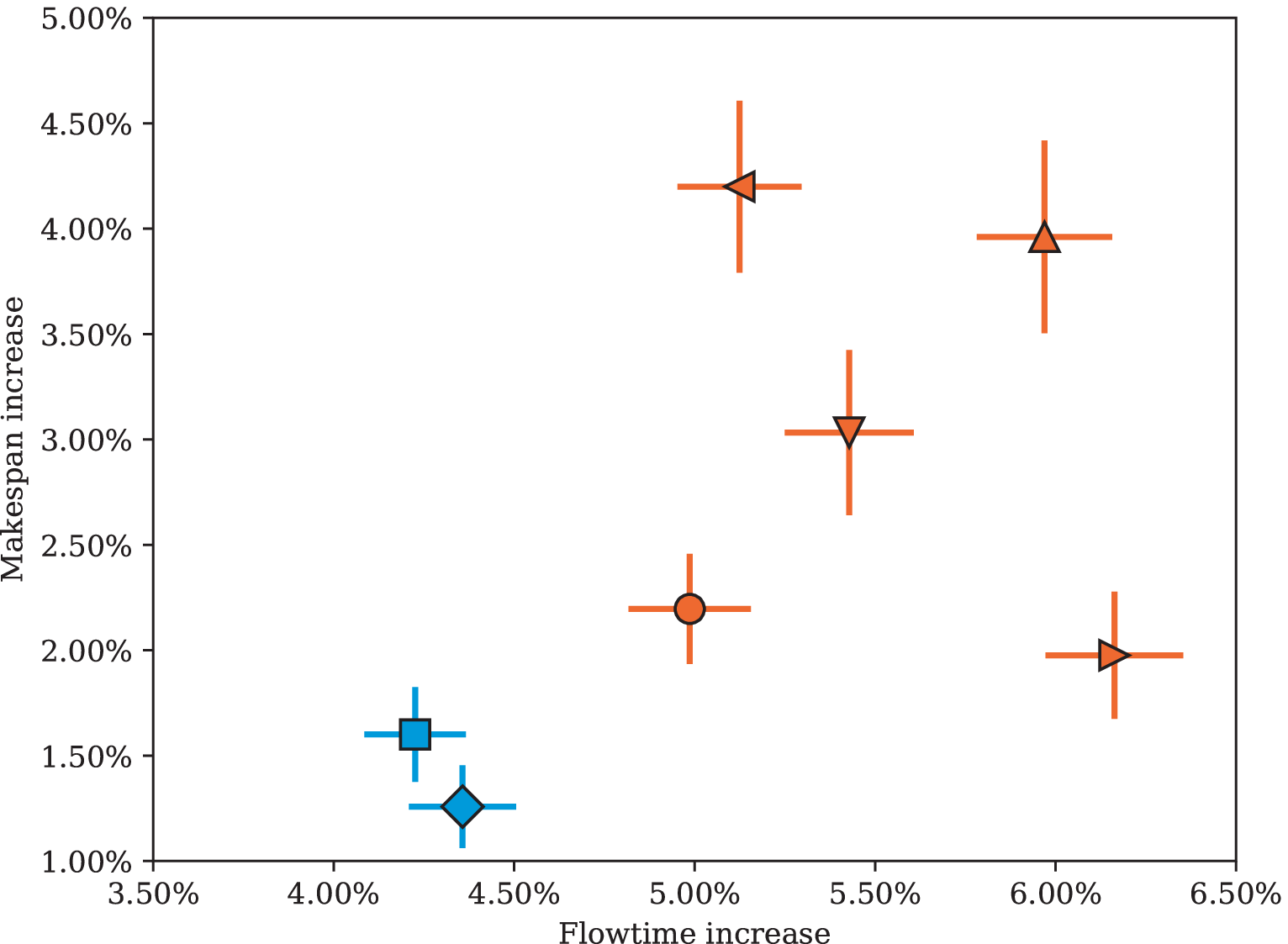}
            \caption{Maze-1}
        \label{fig:p_d}
    \end{subfigure} 
    \begin{subfigure}{0.32\textwidth}
            \centering
            \includegraphics[width=0.95\textwidth]{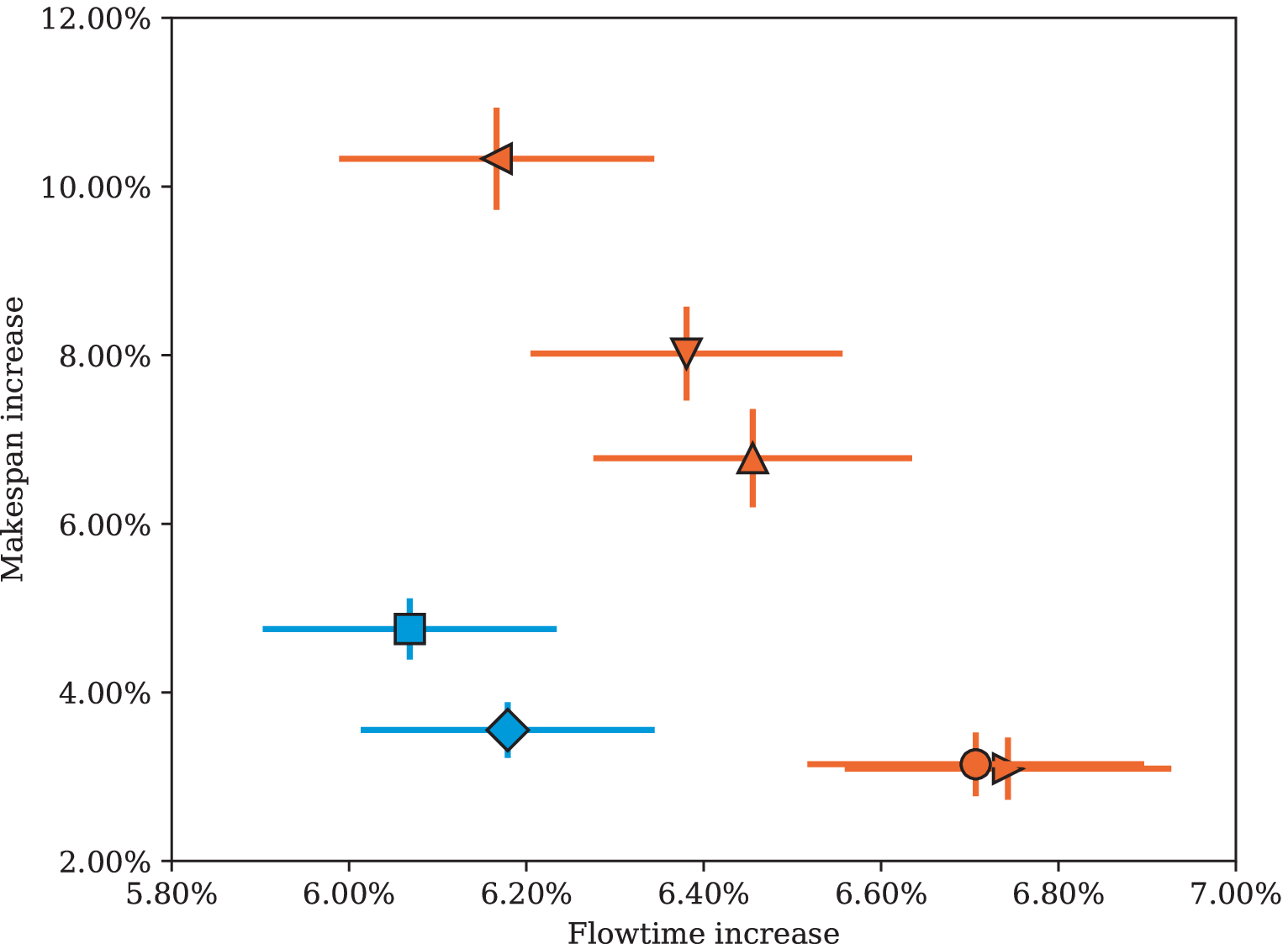}
            \caption{Tunnel}
        \label{fig:p_e}
    \end{subfigure}
    \begin{subfigure}{0.32\textwidth}
            \centering
            \includegraphics[width=0.95\textwidth]{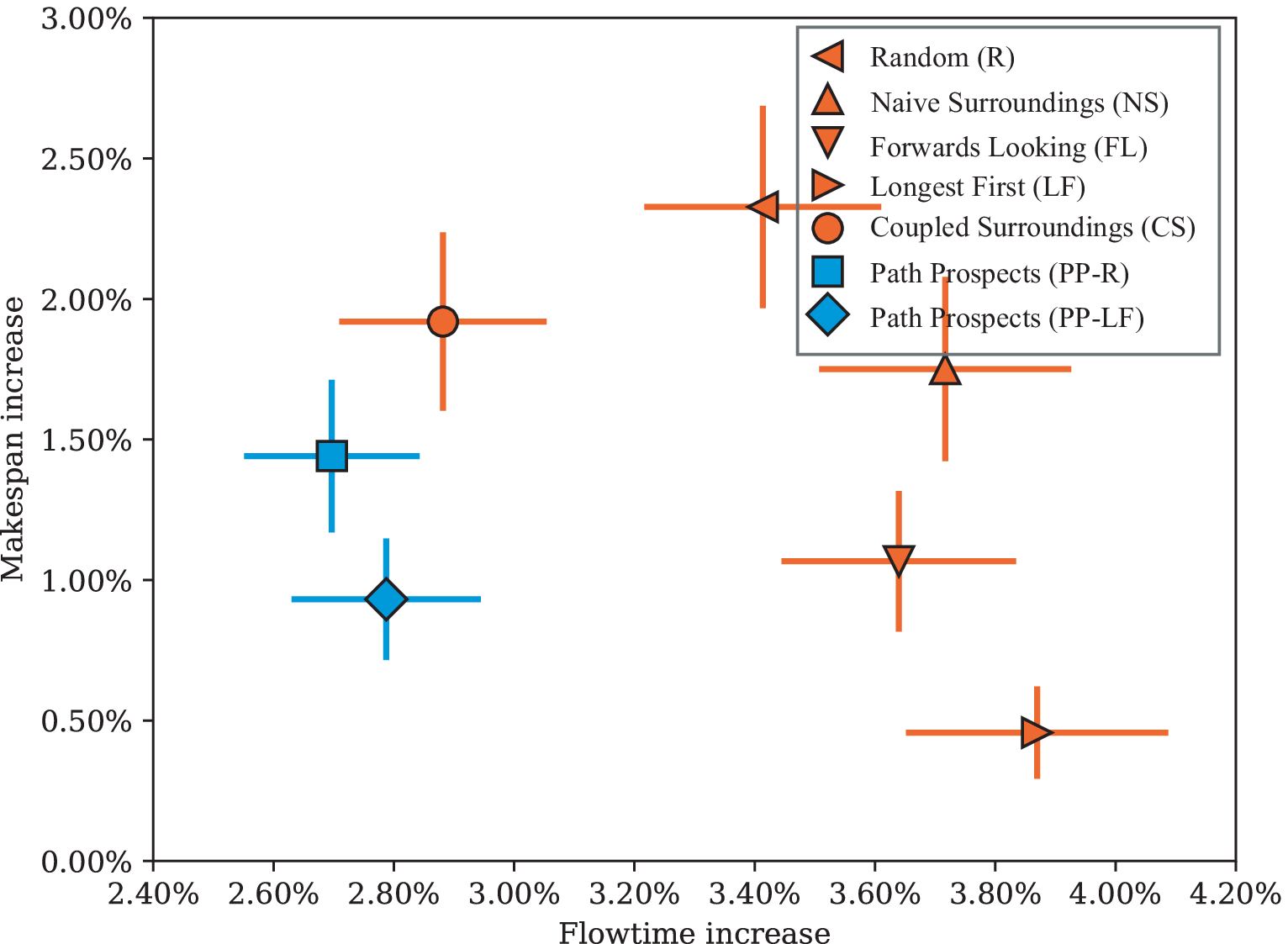}
            \caption{Maze-2}
        \label{fig:p_f}
    \end{subfigure}
    \caption{Experiment \textbf{S1}. Percentage increase over the ideal flowtime and ideal makespan, for the seven variant prioritization heuristics. We show a 95\% confidence interval. Blue nodes correspond to path prospect heuristics, red nodes represent the alternate benchmarks.}
    \label{fig:pareto_results}
\end{figure*}


\subsection{Benchmarks} 
\label{sec:benchmark}

In order to test the efficacy of our prioritization method, we perform an ablation analysis.
The aim of this ablation study is to identify the efficacy of our proposed \emph{path prospects} heuristic by isolating its two key components: \emph{(i)} the spatial area within which it is applied, and \emph{(ii)} the consideration of the robot-environment coupling.
To this end, we implement seven variant schemes for online decentralized prioritization.
Four of these schemes incorporate state-of-the art heuristics, two of the schemes represent our proposed method, and the final scheme incorporates a random rule:

\textbf{(1)} \textbf{\emph{Naive Surroundings (NS):}} 
This prioritization heuristic follows the idea in~\cite{clark2002applying}, whereby robots with the most cluttered surrounding workspace are prioritized. Our implementation of this method counts the number of original obstacles in $\mathcal{O}$ within a range $z=30$ (which corresponds to the best performing range found via grid-search). This variant does not consider the coupling between robot mobility and the environment, and we term it \emph{naive}. We break ties by prioritizing robots with longest remaining paths. 
\textbf{(2)} \textbf{\emph{Coupled Surroundings (CS):}} 
This prioritization heuristic also follows~\cite{clark2002applying}, yet we adapt it to consider the coupling between robot mobility and the environment, whereby effective obstacles in $\mathcal{\tilde{O}}$ are counted (instead of original obstacles). When robot priorities are equal, we tie-break by giving a higher priority to the robot that has the longest remaining path.
\textbf{(3)} \textbf{\emph{Longest First (LF):}}
This method prioritizes the robot that has the longest remaining path to its goal, which corresponds to the heuristic used in~\cite{vandenberg:2005}. When robot priorities are equal, we tie-break by giving a random priority order.
\textbf{(4)} \textbf{\emph{Forwards Looking (FL):}}
We consider a \emph{naive} approach that disregards the coupling of robot mobility and the environment. It is naive in that it uses original obstacles in $\mathcal{O}$ instead of obstacles in $\mathcal{\tilde{O}}$ to compute the number of path options.
The number of path options is considered in the area that contains paths with a cost less than the cost of the currently longest path known, as specified by~Alg.~\ref{alg:forwards}. We tie-break this method by prioritizing robots with the longest remaining paths.
\textbf{(5)} \textbf{\emph{Path Prospects (PP-R):}}
This method implements our path prospect algorithm. We tie-break randomly.
\textbf{(6)} \textbf{\emph{Path Prospects (PP-LF):}}
This method implements our path prospect algorithm. We tie-break with longest-first.
\textbf{(7)} \textbf{\emph{Random (R):}}
Finally, we also implement a prioritization rule that randomly assigns the priority order.

\subsection{Results}
For each run, we record the flowtime, makespan, and whether the run succeeded (i.e., all robots reached their goal locations). 
First, we evaluate the seven algorithm variants by computing two performance metrics: we consider the percent increase in makespan and flowtime, over the ideal makespan and flowtime, respectively, that assumes a collision-free world without robot interactions. 
Figure~\ref{fig:pareto_results} shows a scatter plot of these values, for each base environment in \textbf{S1}.
On all plots, our two proposed methods \textbf{\emph{PP-R}} and \textbf{\emph{PP-LF}} lie on the empirical Pareto front (i.e., lowest values over both dimensions).  
Compared to \textbf{\emph{LF}}, our method provides a valuable trade-off when flowtime is important. When comparing  \textbf{\emph{PP-R}} and \textbf{\emph{LF}} to \textbf{\emph{CS}}, the panels show that \textbf{\emph{CS}} incurs a loss of performance in makespan or flowtime performance, or both. This shows that the \emph{area} within which path options are considered is important. The panels also show that our two methods consistently outperform the naive variant, \textbf{\emph{FL}}, which uses the same area for computing path prospects (i.e., forwards vertices), but disregards the robots' mobility within this area. This demonstrates the importance of considering the coupling between the robot and its environment.

Figure~\ref{fig:success} shows the success rates for the seven algorithms. The results show that success rates increase significantly with heuristics that explicitly account for the robot-environment coupling. The highest success rates are achieved by our two methods, \textbf{\emph{PP-R}} at 95.7\% and \textbf{\emph{PP-LF}} at 94.1\%.

Figure~\ref{fig:results_largemap} shows the percent increase in makespan and flowtime for the large map used in our second experiment  (\textbf{S2}). The results corroborate the results obtained over the smaller environments; the experiment also demonstrates the applicability of our method to large numbers of agents. 


\begin{figure}[tb]
    \centering
    \includegraphics[width=0.7\columnwidth]{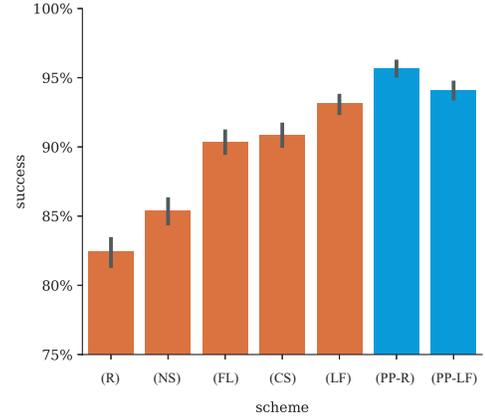}
    \caption{Success rates for the seven variant heuristics, averaged over all environments and communication radii. The errorbars represent 95\% confidence intervals. The two columns to the right (blue) correspond to our path prospect heuristics that account for the robot-environment coupling.}
    \label{fig:success}
\end{figure}

\begin{figure}
\centering
\includegraphics[width=0.35\textwidth]{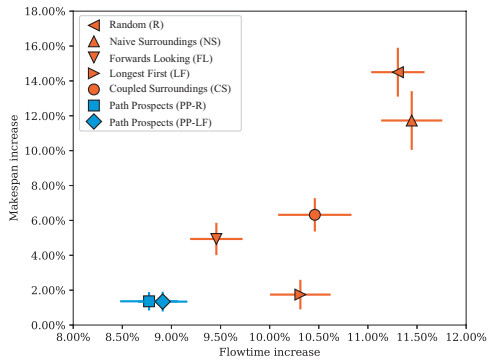}
\caption{Experiment \textbf{S2} obtained on a large map of size 150$\times$150 with 100 robots. Percentage increase over the ideal flowtime and ideal makespan, for the seven variant prioritization heuristics. We show a 95\% confidence interval. Blue nodes correspond to path prospect heuristics, red nodes represent the alternate benchmarks.
\label{fig:results_largemap}}
\end{figure}



\subsection{Discussion}
We presented a method for dynamic prioritized path planning for teams of robots. Our decentralized, decoupled planning algorithm provides a deadlock-free means of negotiating path plans among robots, and uses a prioritization heuristic that is based on $\mathbb{Z}_2$ coefficient homology classes, which quantifies a robot's number of available path options. 
This heuristic makes use of two key components. First, it estimates the number of path options available to a robot for it to reach its goal. Second, it defines an area within which these path options are computed. We compared our method to five alternate heuristics. Although our objectives (minimum makespan and minimum flowtime) cannot be simultaneously optimized, we showed that our method strikes the best balance, and lies on the empirical Pareto front of these considered benchmarks. Future work will consider the application of this method to a wider variety of configuration spaces (beyond grid-worlds), in  3-dimensional, continuous domains.

The use of a prioritization heuristic is efficient; this is particularly true when the technique eschews the need to evaluate all possible total priority orderings. In our presented decentralized algorithm, each robot is able to independently compute its own priority, since this is an absolute value that depends only on the robot itself and the environment it is moving in. Hence, we reduce the otherwise exponential computational complexity to a linear one (in the centralized case); the decentralized solution is distributed, and depends linearly on the number of neighboring robots. 
In this case-study, we not only expose the tight coupling between a robot's mobility and its surrounding environment, but also, we show that by explicitly considering this relationship, we are able to improve performance.


\section{Acknowledgement}
We gratefully acknowledge the support of ARL grant DCIST CRA W911NF-17-2-0181. This work incorporates results from the research project “Co-Evolving Built Environments and Mobile Autonomy for Future Transport and Mobility” funded by the Centre for Digital Built Britain, under InnovateUK grant number RG96233.

{
\footnotesize
\bibliographystyle{abbrv}
\bibliography{Bibliography}
}

\end{document}